\documentclass[letterpaper]{article}

\usepackage{aaai20}
\usepackage{times}
\usepackage{helvet}
\usepackage{courier}
\usepackage{amsthm}
\usepackage{amssymb}
\usepackage{amsmath}
\usepackage{amsfonts}
\usepackage{appendix}

\usepackage{algorithm}
\usepackage[noend]{algpseudocode}

\usepackage{amsfonts}
\usepackage{verbatim}
\usepackage{amsmath}
\usepackage{amsthm}
\usepackage{amssymb}
\usepackage{graphicx} 
\usepackage{caption} 
\usepackage{subcaption}
\usepackage[notcite,notref, final]{showkeys}
\usepackage[all,cmtip]{xy}

\newtheorem{theorem}{Theorem}

\newtheorem{proposition}[theorem]{Proposition}

\newtheorem{corollary}[theorem]{Corollary}

\newtheorem{definition}[theorem]{Definition}

\usepackage{tikz}
\usetikzlibrary{arrows.meta}
\usepackage{enumitem}
\newcommand{\Pa}{\text{Pa}}

\newcounter{int}
\makeatletter
\newcommand{\citen}[1] {\setcounter{int}{0}\@for\tmp:=#1\do{%
\ifnum \value{int}>0; \fi%
\setcounter{int}{1}%
\citeauthor{\tmp} \shortcite{\tmp}}}
\makeatother
\makeatletter
\newcommand{\citenp}[1]{\setcounter{int}{0}\@for\tmp:=#1\do{%
\ifnum \value{int}>0; \fi%
\setcounter{int}{1}%
\citeauthor{\tmp}, \citeyear{\tmp}}}
\makeatother

\begin{document}
%
\title{Deception through Half-Truths}
\author{Andrew Estornell, Sanmay Das, Yevgeniy Vorobeychik\\Computer Science \& Engineering, Washington University in St.~Louis\\\{aestornell,sanmay,yvorobeychik\}@wustl.edu}
\maketitle

\begin{abstract}
Deception is a fundamental issue across a diverse array of settings, from cybersecurity, where decoys (e.g., honeypots) are an important tool, to politics that can feature politically motivated ``leaks'' and fake news about candidates.
Typical considerations of deception view it as providing false information.
However, just as important but less frequently studied is a more tacit form where information is strategically hidden or leaked.
We consider the problem of how much an adversary can affect a principal's decision by ``half-truths'',
that is, by masking or hiding bits of information, when the principal is oblivious to the presence of the adversary. The principal's problem can be modeled as one of predicting future states of variables in a dynamic Bayes network, and we show that, while theoretically the principal's decisions can be made arbitrarily bad, the optimal attack is NP-hard to approximate, even under strong assumptions favoring the attacker. 
However, we also describe an important special case where the dependency of future states on past states is additive, in which we can efficiently compute an approximately optimal attack. Moreover, in networks with a linear transition function we can solve the problem optimally in polynomial time. 
\end{abstract}

\section{Introduction}
For better or for worse, deception is ubiquitous.
It can be benign, but just as often deception is used to deliberately mislead.
Commonly, the means of deception can be viewed as outright lies or misinformation.
This is certainly the case with fake news and false advertising, as well as phishing emails, and it is also the case for honeypots, even though here deception is used to help network security, rather than for a nefarious purpose.
However, a more subtle means of deception involves strategically hiding information.
For example, misleading advertising about a drug may omit important information about its side-effects, and we may effectively protect a system against classes of attacks by strategically deciding what is public about it, such as a Windows computer publicizing a Safari browser, but not the OS, to make it appear it's running Mac OS X.

Theoretical studies of deception typically leverage games of incomplete information, where deception takes the form of signaling misinformation about private state~\cite{Carrol11,Pawlick15}, for example, advertising an incorrect configuration of computing devices (e.g., a Windows machine advertising as Linux)~\cite{Schlenker18}, or warning that there may be inspections when no inspectors are present~\cite{Xu16}.
We take a 
different perspective.
Specifically, we start with a decision-maker (the \emph{principal}) who makes decisions under uncertainty based on limited evidence.
To formalize this setting, we consider a two-stage dynamic Bayes network in which the principal observes a partial realization of the first stage, and makes a prediction (i.e., derives a posterior) about the second stage.
We study the extent to which such a decision-maker is susceptible to deception through \emph{half-truths}---that is, through an adversarial masking of a subset of first-stage variables, with the assumption that the principal is oblivious to the adversarial nature of this masking (for example, the individual is unaware, or fails to take into account, that it is performed adversarially).

While it may at first blush be puzzling how a rational Bayesian observer would be oblivious to the presence of an adversary, situations of this kind in fact abound.
Consider algorithmic trading as one example.
When order book information became available, it gave rise to numerous sophisticated machine learning methods aiming at taking advantage of this additional information~\cite{Nevmyvaka06,Nevmyvaka13}.
However, many such approaches proved to be vulnerable to order book spoofing attacks~\cite{Wang18}.
Another example is autonomous driving.
Despite a number of illustrations of attacks on state-of-the-art sophisticated AI-based perception algorithms~\cite{Boloor19,Eykholt18,Sharif16,Vorobeychik18book}, standard autonomous driving stacks, such as Autoware~\cite{Autoware} and Apollo~\cite{Apollo} are largely devoid of any techniques for robust perception.

Our first observation is that in our setting half-truths (that is, adversarial masking of observations) can lead to arbitrarily wrong beliefs.
This is self-evident with lies, but surprising when we can only mask observations.
However, we show that the problem of optimally choosing such a mask is extremely hard: in general, it is inapproximable to any polynomial factor.
Next, we study an important restricted family of Bayes networks in which transition probabilities of nodes depend on the sum of the parents.
This is a natural model if we consider, for example, opinion diffusion through social influence.
For example, suppose that each variable represents whether an individual likes a particular candidate in an election.
The opinions in the second stage would correspond to the impact of social influence, where parents of a node are their social network neighbors.
Our model means that a node's view depends on the number of their neighbors who like the candidate.
In this \emph{additive} model, we show that the problem does not admit a PTAS even when nodes have at most two parents.
However, we exhibit two algorithmic approaches for solving this variant:
the first an $n$-approximation algorithm, the second a heuristic (which admits no performance guarantees).
Our experiments show that the combination of the two yields good performance in practice, even while each is limited by itself.
Finally, we show that when temporal dependency is linear, we can find an optimal mask in polynomial time.

\noindent{\bf Related Work } A number of prior efforts study deception, many in the context of cybersecurity.
Among the earliest is work by~\citeauthor{Cohen03}~\shortcite{Cohen03},
who formalize deception as guiding attackers
through (a benign part of) the attack graph.
Recent \emph{qualitative} studies of deception~\cite{Almeshekah16,Stech16} offer additional insights, but do not provide mathematical modeling approaches.
A series of mathematical formalizations of deception in
cyber security have also been proposed~\cite{Carrol11,Greenberg82,Ettinger10,Pawlick15,Xu16}, but these tend to model \emph{static} scenarios and misinformation, rather than information hiding.
Several other mathematical models address allocation of honeypots, which is a common means for deceiving cyber attackers~\cite{Kiekintveld15}.
Recently, deception has also been considered as a security game in which a defender chooses a deceptive presentation of system configuration to an attacker~\cite{Schlenker18}, but without considering half-truths or structured information representation such as a DBN.

Another relevant stream of research is that on \emph{information design} \cite[e.g.]{Rayo2010Sender}. In the commonly studied Bayesian persuasion model \cite{kamenica2011bayesian}, one considers a signaling game between a sender and a receiver, where the sender has the ability to acquire superior information to the receiver, and the receiver makes a decision that yields (state-dependent) utilities for both. 
The key question concerns the design of the optimal signal structure.
This area has recently received attention from both the algorithmic perspective (how hard is the sender's problem under different assumptions \cite{dughmi2016algorithmic}) and in various applications, for example pricing \cite{Shen2018Closed}, auction design \cite{li2019signal}, and security games \cite{Rabinovich2015information}. 
Our work is distinct in that it assumes an oblivious principal, but effectively considers signals which have combinatorial structure.
\section{Preliminaries} 

Consider a collection of binary variables $\mathbf{X} = \{X_1, ..., X_n\}$.
We define a 2-stage dynamic Bayes network over these, using superscripts to indicate time steps (0 and 1).
Specifically, we assume that each $X_i^0$ is unconditionally independent and for each $X_i^0$, let $\mathbb{P}(X_i^0=1) = p_i$.
Moreover, each $X_i^1$ has a set of parent nodes, $\Pa(X_i^1)\subset\mathbf{X}^0$ (we only allow inter-stage dependencies to simplify discussion), and for each $X_i^1$, define $\mathbb{P}(X_i^1=1|Pa(X_i^1))$ as the probabilistic relationship of the associated variable with its parents (variables it depends on) from stage 0.
We will denote the realized values of these random variables in lower case: that is, the realization of a random variable $X_i^t$ is $x_i^t$.


We use this structure to define an interaction between an attacker and a myopic observer (who we also call the \emph{principal}).
In particular, consider an observer who observes a partial realization of stage-0 variables, and aims to predict (in a probabilistic sense) the values of variables in stage 1.
This high-level problem is a stylized version of a broad range of decision problems, such as voting behavior.
Examples include observing candidate promises, personalities, and past voting record, to predict what they would do once elected; observing infection status for a collection of individuals on a social network, and aiming to predict who will be infected in the future; and so on.
We assume that the observer is myopic in the sense that they use standard Bayesian reasoning about posterior probabilities conditional on their observations of stage-0 realizations.
However, we specifically study a situation in which a malicious party adversarially masks a subset of stage-0 realizations (having first observed them).
We denote the masked posterior by $\mathbb{P}(X_i^1 = 1| \Pa(X_i^1)\setminus \eta)$, where $\eta$ is a binary vector with $\eta_i = 1$ whenever the realization of $X_i^0$ is not observed (because it is masked).
We assume that all the stage-0 realizations that are not masked are observed by the principal.
Let $\bold{X}^1$ denote the random vector distributed according to $\mathbb{P}(X_i^1 = 1| \Pa(X_i^1)$ (the full set of its parents from $\bold{X}^0$), while $\bold{X}^1_\eta$ is a random vector distributed according to $\mathbb{P}(X_i^1 = 1| \Pa(X_i^1)\setminus \eta)$.
More precisely, the sequence of the interaction is as follows:
\begin{enumerate}
    \item Nature generates a vector $\bold{x}^0 = \langle x_1^0, ..., x_n^0 \rangle$ defining the outcomes of $\bold{X}^0$ according to its prior distribution $p$.
    \item The attacker observes $\bold{x}^0$ and may choose up to $k$ outcomes to hide from the observer.  This decision is captured by the mask $\eta$.
    \item The observer observes the partially realized state of $\bold{X}^0$ after applying the mask $\eta$, and makes a prediction about $\bold{X}^1$ (which we capture by the distribution of $\bold{X}^1_\eta$).
    \item Nature then yields the realization of $\bold{x}^1 = \langle x_1^1, ..., x_1^1 \rangle$ according to the posterior distribution of $\bold{X}^1$.
\end{enumerate}

To understand the consequence of adversarial ``half-truths'' of this kind, we consider two problems faced by the adversary: targeted and untargeted attacks.
Specifically, let the two random vectors, $\bold{X}^1$ and $\bold{X}^1_\eta$ also stand for their respective distributions, and let $D(\bold{X}^1, \bold{X}^1_\eta)$ be a statistical distance between the two distributions according to some metric.
In the untargeted case, the adversary's problem is to maximize the distance between the masked and true posterior distributions over the random vector in stage 1:
\begin{equation}
    \label{E:advproblem}
    \max_\eta D(\bold{X}^1, \bold{X}^1_\eta) \quad \mathrm{s.t.:} \quad \sum_i \eta_i \le k.
\end{equation}
In the targeted case, the adversary has some desired distribution, $\mathbf{X}^1_{\alpha}$, and the adversary would like to push the observer's perception as close to this distribution as possible.
We formalize this as
\begin{equation}
    \label{E:targetproblem}
    \min_{\eta}D(\mathbf{X}_{\alpha}^1, \mathbf{X}_{\eta}^1) \quad \mathrm{s.t. :} \quad \sum_i \eta_i \le k.
\end{equation}
Note that in this notation we are suppressing the dependence on the prior, which is implicitly part of any problem instance faced by the adversary.

\section{Half-Truth is as Good as a Lie}

Our first result demonstrates that in a fundamental sense, in our model, there are cases where partially hiding the true current state can lead to arbitrary distortion of belief by a myopic observer.

Recall that the adversary's aim is to maximize statistical distance $D$ between the true posterior distribution over $\bold{X}^1$, and the posterior induced by masking a subset of variables in stage 0, $\bold{X}^1_\eta$.
We now show that for most reasonable measures of statistical distance, we can construct cases in which the adversary can make it arbitrarily large (within limits of the measure itself)---that is, the adversary can induce essentially arbitrary distortion in belief solely by masking some of the observations.



\begin{definition}\label{def:possem}
We say a statistical distance is positive if for any two random variables $A$, $B$ we have $D(A, B) \geq 0$.
\end{definition}
Note that any distance metric, or probabilistic extension of a distance metric, fits the definition of positive symmetric. 

\begin{theorem}\label{prop:arb}
Suppose the attacker's objective is to maximize some positive statistical distance $D$. Let $\bold{A}$ and $\bold{B}$ be any vectors of binary random variables, then there exists some sequence of dynamic Bayes networks such that 
$$\lim_{n\rightarrow\infty} \big(\mathbb{E}_{\bold{X}^0}\big[\max_{\eta} D(\bold{X}^1, \bold{X}^1_{\eta})\big]\big) = \lim_{n\rightarrow\infty}\big(\max_{\bold{A}, \bold{B}}D(\bold{A}, \bold{B})\big)$$
\end{theorem}

\begin{proof}
Let $\bold{A}, \bold{B}$ be the vectors of binary random variables for which $D(\bold{A}, \bold{B})$ attains its maximum value, with respect to $n$. Then $\bold{A} = \langle A_1, ..., A_n \rangle$, $\bold{B} = \langle B_1, ..., B_n \rangle$ and each variable has prior $\mathbb{P}(A_i = 1 ) = a_i$, $\mathbb{P}(B_i = 1) = b_i$.
Let $\boldsymbol{X}^0 \rightarrow \boldsymbol{X}^1$ define a dynamic Bayes network on $n$ variables. For all $1\leq j \leq n$, let $\Pa(X_j^1) = \{X_i^0: 1\leq i \leq n\}$. 
That is, all nodes in layer 0 are parents of every node in layer 1. 
Define the probability distributions over $\boldsymbol{X}^0$ and $\boldsymbol{X}^1$ by the following: $\forall X_i^0\in \boldsymbol{X}^0$, $\mathbb{P}(X_i^0 = 1) = \epsilon$. Next, $\forall X_i^1\in\boldsymbol{X}^1$, $\mathbb{P}(X_i^1 = 1| \exists x_j^0 = 1) = b_i$ and $\mathbb{P}(X_i^1 = 1| \nexists x_j^0 = 1) = a_i$. 


For each $n$ we will consider the value of $D(\bold{X}^1, \bold{X}_{\eta}^1)$ under three types of events that could occur with respect to the possible outcomes, $\bold{x}^0$ of $\boldsymbol{X}^0$, the adversary's budget $k$, and the adversary's choice of which nodes to hide conditional on $\bold{x}^0$. Each of these settings admits a unique type of optimal play from the adversary. Specifically
\begin{itemize}
    \item (1) $\sum_{X_j^0}x_j^0 = 0$. In this case the adversary will hide $k$ random nodes since all outcomes are 0.
    \item (2) $\sum_{X_j^0}x_j^0 = m \leq k$. In this case the adversary will hide only the $m$ nodes whose outcomes are 1.
    \item (3) $\sum_{X_j^0}x_j^0 = m > k$. In this case the adversary will hide nothing.
\end{itemize}
In events of type (1) when there is no mask $\bold{X}^1 = \bold{A}$. When a mask $\eta$ is employed,
$\bold{X}_{\eta}^1 = \bold{B}$ with probability $1-(1-\epsilon)^k$, and $\bold{X}_{\eta}^1 = \bold{A}$ with probability $(1-\epsilon)^k$.
Thus, in this setting, 
\begin{align*}
\mathbb{E}\big[ D(\bold{X}^1, \bold{X}_{\eta}^1 )\big] =&\big(1 - (1-\epsilon)^k\big)D(\bold{A}, \bold{B}) \\
&+ (1-\epsilon)^k D(\bold{A}, \bold{A})
\end{align*}
Events of this type occur with probability $(1-\epsilon)^n$.

In events of type (2), without $\eta$ we have $\bold{X}^1 =  \bold{B}$. Since $m \leq k$ and all nodes with outcome 0 are hidden. 
Thus, in light of $\eta$ we have $\bold{X}_{\eta}^1 = \bold{A}$ with probability $(1 - \epsilon)^m$, and $\bold{X}_{\eta}^1 = \bold{B}$ with probability $1 - (1-\epsilon)^m$.
Therefore, the expected value in this setting is  
\begin{align*}\mathbb{E}\big[ D(\bold{X}^1, \bold{X}_{\eta}^1)\big] = &\big(1 - (1-\epsilon)^m\big)D(\bold{B}, \bold{B}) \\
&+ (1-\epsilon)^m D(\bold{B}, \bold{A})
\end{align*}
Events of this type occur with probability $\binom{n}{m}\epsilon^m(1-\epsilon)^{n-m}$ for each $m \leq k$.

In events of type (3) there are more nodes yielding 1 in layer 0 than the adversary is capable of hiding. So $\bold{X}^1 = \bold{X}_{\eta}^1 = \bold{A}$. Events of this type occur with probability $\binom{n}{m}\epsilon^m(1-\epsilon)^{n-m}$ for each $m > h$.

For notational convenience, and without loss of generality, we will reorder the nodes in $\boldsymbol{X}_n^0$ after the observations are made by the adversary, such that for $0\leq j \leq m$, $x_j^0 = 1$.
Suppose $k = n$, similar analysis holds for any constant fraction of $n$. Since $D$ is positive symmetric we have,
\begin{align*}
 & \mathbb{E}_{\bold{X}^0}\big[\max_{\eta} D(\bold{X}^1, \bold{X}^1_{\eta})\big] \\
 \geq & D(\bold{A}, \bold{B})\big(1-(1-\epsilon)^n\big)(1-\epsilon)^n \\
\qquad &+ D(\bold{B}, \bold{A})\bigg(\sum_{m=1}^n\binom{n}{m}\epsilon^m(1-\epsilon)^{n-m}(1-\epsilon)^m\bigg)
\end{align*}


Using the binomial identities we can reduce the above equation to form
\begin{align*}
&D(\mathbf{A}, \mathbf{B})\big(1 - (1-\epsilon)^n\big)(1-\epsilon)^n \\
+ &D(\mathbf{B}, \mathbf{A})(1-\epsilon )^n \left((\epsilon +1)^n-1\right)
\end{align*}
Thus, since both terms in the above sum are positive, it remains only to be shown  for $\epsilon = \frac{\log(n)}{n}$,
\begin{align*}
(1-\epsilon )^n \left((\epsilon +1)^n-1\right) \rightarrow 1 \text{ as } n\rightarrow \infty
\end{align*}
This limit can be evaluated as follows
\begin{align*}
 = &\lim_{n \rightarrow \infty}\left(1-\frac{\log (n)}{n}\right)^n \left(\left(1 +\frac{\log (n)}{n}\right)^n-1\right)
\end{align*}
Using a slight variation to the identity $\lim_{n\rightarrow \infty}(1+\frac{a}{n})^{c n} = e^{ac}$, we can obtain that this limit does in-fact converge to 1.
Thus giving the desired result that 
\[
\lim_{n\rightarrow\infty} \big(\mathbb{E}_{\bold{X}^0}\big[\max_{\eta} D(\bold{X}^1, \bold{X}^1_{\eta})\big]\big) = \lim_{n\rightarrow\infty}\big(\max_{\bold{A}, \bold{B}}D(\bold{A}, \bold{B})\big)
\]

\end{proof}



\section{Computational Complexity of Deception by Half-Truth}

Let $\bold{X^0} \rightarrow \bold{X^1}$ define a dynamic Bayes network over a set of $n$ binary random variables. Let $\bold{x}^0$ be a binary vector describing the realized outcomes of $\bold{X}^0$. 

In the remainder of the paper, we restrict attention to particular distance metrics of the form: 
\begin{align*}
\text{untargeted: }& D(\bold{X}^1,\bold{X}^1_\eta) = \mathbb{E}\big[||\bold{X}^1 - \bold{X}^1_{\eta}||_p\big] 
\\
\text{targeted: }& D(\bold{X}^1,\bold{X}^1_\eta) = \mathbb{E}\big[||\bold{X}_{\alpha}^1 - \bold{X}^1_{\eta}||_p\big]
\end{align*}
where the expectation is with respect to the product distribution of the two random variables and $p\in \mathbb{N} \cup \{\infty\}$.
These are natural distances in the context of random variables, and correspond to the
Lukaszyk-Karmowski metric (LKM) of statistical distance between the distributions.
We call the resulting problems (of computing the optimal mask given a prior and a realization of variables at layer 0) \emph{Deception by Bayes Network Masking (DBNM)} for the untargeted case, and \emph{Targeted Deception by Bayes Network Masking (TDBNM)} for the targeted case.
We now show that this problem does not even admit a polynomial factor approximation for any $p$.



\begin{theorem} \label{P=NP}
 If DBNM has a deterministic, polynomial-time, polynomial approximation, for any value of $p$, then P=NP.
\end{theorem} 

\begin{proof} Suppose that there exists a deterministic, polynomial factor, polynomial time approximation of DBNM. 
We will show that under this assumption 
SAT can be solved in polynomial time.
Consider an instance of SAT defined by a set of Boolean variables $B$ and a Boolean function $\Phi$, whose terms are the elements of $B$. 
The objective is to determine if there exists an assignment of the variables in $B$ such that $\Phi$ evaluates to 1.
An arbitrary instance of SAT can be encoded into DBNM in the following manner.
Let $\bold{X^0} = B$, $\Pa(X_1^1) = \bold{X^0}$, and define $\mathbb{P}(X_1^1 = 1| \Pa(X_1^1)) = \Phi$ (that is, $X_1^1 = 1$ if and only if the formula $\Phi$ evaluates to true). For all other $j \neq 1$, $\mathbb{P}(X_j^1 = 1 | \Pa(X_j^1)) = 0$. Lastly, set each prior $\mathbb{P}(X_i^0 = 1) = \frac{1}{2^{2n}}$ and set $\bold{x}^0 = \langle 1, 1, ..., 1 \rangle$. 

In the case that $b = 1$, $\forall b \in B$,  yields $\Phi = 0$, the objective of the attacker is to select a mask $\eta$ that maximize the value of $\mathbb{P}(X_1^1 = 1 | \bold{x}^0 \setminus \eta)$. For a given mask $\eta$, $\bold{y}_{\eta}^0$ be any outcome that agrees with $\bold{x}^0$ on all in $\bold{X^0} \setminus \eta$, i.e. $x_i = y_{\eta, i}^0$ for all $X_i^0\notin\eta$. 
Let
\[ 
a_{\bold{y}_{\eta}^0} = ||\bold{x}^0 - \bold{y}_{\eta}^0||_1
\]
Then, for any $\eta$ we have,
\begin{align*}
& \mathbb{P}(X_{\eta, 1}^1 = 1 | \bold{x^0}) = \sum_{\bold{y}_{\eta}^0}\mathbb{P}(\bold{y}_{\eta}^0) \mathbb{P}(X_1^1 = 1| \bold{y}_{\eta}^0) \\
= &\sum_{\bold{y}_{\eta}^0}\mathbb{P}(X_1^1 = 1 | \bold{y}_{\eta}^0) (1-\frac{1}{2^{2n}})^{a_{\bold{y}_{\eta}^0}}(\frac{1}{2^{2n}})^{|\eta| - a_{\bold{y}_{\eta}^0}}
\end{align*}

A certificate for the SAT instance can be generated via assigning $b_i = 1$ if $X_i^0\notin \eta$ and $b_i = 0$ if $X_i^0\in \eta$.
To see that this certificate is valid, consider two cases on $\eta$.
The first being, $\eta$ corresponds to an assignment of $B$ yielding $\Phi = 0$, and the second being when the assignment gives $\Phi = 1$.

In the first case, let $\bold{y}_{\eta}^{'0}$ be the $\bold{y}_{\eta}^0$ outcome such that $y_{\eta, i}^0 = 0$ for all $X_i^0\in \eta$ and $y_{\eta, j}^0 = 1$ for all $X_j^0\notin\eta$. 

Then, since $\mathbb{P}(X_1^1 = 1 | \bold{y}_{\eta}^{'0}) = 0$, we have

\begin{align*}
 & \sum_{\bold{y}_{\eta}^0}\mathbb{P}(X_1^1 = 1 | \bold{y}_{\eta}^0) (1-\frac{1}{2^{2n}})^{a_{\bold{y}_{\eta}^0}}(\frac{1}{2^{2n}})^{|\eta| - a_{\bold{y}_{\eta}^0}} \\
 = &\sum_{\bold{y}_{\eta}^0 \neq \bold{y}_{\eta}^{'0} }\mathbb{P}(X_1^1 = 1 | \bold{y}_{\eta}^0) (1-\frac{1}{2^{2n}})^{a_{\bold{y}_{\eta}^0}}(\frac{1}{2^{2n}})^{|\eta| - a_{\bold{y}_{\eta}^0}}
\end{align*}

Note that for each $\bold{y}_{\eta}^0 \neq \bold{y}_{\eta}^{'0}$, $|\eta| - a_{\bold{y}_{\eta}^0}\geq 1$.
Thus, 
\begin{align*}
 = &\sum_{\bold{y}_{\eta}^0 \neq \bold{y}_{\eta}^{'0} }\mathbb{P}(X_1^1 = 1 | \bold{y}_{\eta}^0) (1-\frac{1}{2^{2n}})^{a_{\bold{y}_{\eta}^0}}(\frac{1}{2^{2n}})^{|\eta| - a_{\bold{y}_{\eta}^0}} \\
 \leq & \sum_{\bold{y}_{\eta}^0 \neq \bold{y}_{\eta}^{'0}}\frac{1}{2^{2n}} \leq 2^n(\frac{1}{2^{2n}})= \frac{1}{2^n}
\end{align*}
Therefore, if the adversary selects a mask that does not correspond to a satisfying assignment for $\Phi$, its utility is at most $\frac{1}{2^n}$.

The next case to consider is when the adversary selects a a mask which induces $\Phi =1$.
In this case, we have 
\begin{align*}
&\sum_{\bold{y}_{\eta}^0  \neq \bold{y}_{\eta}^{'0} }\mathbb{P}(X_1^1 = 1 | \bold{y}_{\eta}^0) (1-\frac{1}{2^{2n}})^{a_{\bold{y}_{\eta}^0}}(\frac{1}{2^{2n}})^{|\eta| - a_{\bold{y}_{\eta}^0}} \\
 &\quad\quad +  \mathbb{P}(X_1^1 = 1 |\bold{y}_{\eta}^{'0})(1 - \frac{2}{2^{2n}})^{a_{\bold{y}_{\eta}^{'0}}} \\
 & \geq  (1 - \frac{1}{2^{2n}})^{a_{\bold{y}_{\eta}^{'0}}} \geq (1 - \frac{1}{2^{2n}})^n 
\end{align*}
Thus, if $\eta$ induces an assignment of $B$ that yields $\Phi = 1$, the adversary utility at least $(1-\frac{1}{2^{2n}})^n$.
Which converges to $1$, from below, faster than an polynomial of $n$.

By these two cases, we know that when $\Phi$ is satisfiable, there exists a mask with value at least $(1 - \frac{1}{2^{2n}})^n$ and that no mask corresponding to $\Phi=0$ can have value greater than $\frac{1}{2^n}$.
In addition to the results of these two cases, we also know that an optimal mask can achieve no more than a value of $1$, since only 1 node in $\bold{X}^1$ has outcomes dependent on $\bold{X}^0$ and any $L_p$ norm applied to a vector with only a single nonzero dimension will evaluate to exactly the value of the dimension. 
Therefore, if a polynomial approximation of the optimal solution were to be given, one could deduce the satisfiability of $\Phi$ based on the value of the mask $\eta$. That is if $V(\eta) \leq \frac{1}{2^n}$, then $\Phi$ is not satisfiable, and if $V(\eta) \geq (1 - \frac{1}{2^{2n}})^n$, then $\Phi$ is satisfiable and $\eta$ gives the satisfying assignment. 

This covers all but the case when $b_i = 1$, $\forall b_i \in B$, yields $\Phi = 1$. In this case, the adversary could return a mask of value arbitrarily close to $0$ even though $\Phi$ has a satisfying assignment. This case is easily remedied by choosing to check the assignment $b_i =1$, $\forall b_i\in B$, before running the approximation.

Under this scheme we could use the polynomial approximation algorithm to determine if a given instance of SAT is satisfiable. Since SAT is NP complete, the existence of such an approximation algorithm would imply that P = NP.
\end{proof}

Next, we show that this inapproximability obtains even if we consider randomized algorithms.
\begin{theorem}\label{RP=NP}
 If DBNM has a randomized polynomial factor approximation with constant probability, for any $p$, then PR = NP.
\end{theorem}
\begin{proof} Using the previous construction from SAT to DBNM. If there existed an algorithm that could produce a  polynomial factor approximation of the constructed instance of DBNM with some constant probability $p \in (0, 1)$, then the same line of reasoning in the above proof yields a polynomial time algorithm that can determine if a true instance of SAT is satisfiable with probability at $p\in (0, 1)$. 
This algorithm could then be run $\frac{1}{p}$ times to obtain a success rate of $ 1 - (1 - p)^{\frac{1}{p}} \geq 1 - \frac{1}{e} \geq \frac{1}{2}$.
Moreover, the algorithm would never falsely identify a non-satisfiable instance as satisfiable. 
The existence of such an algorithm would imply that SAT $\in$ RP, and since SAT is NP-complete and RP is closed under L-reductions, this would also imply that RP = NP.
\end{proof}

Finally, we extend the hardness results above to the targeted version of our problem.
\begin{corollary}
If TDBNM has a deterministic polynomial time, polynomial approximation, or a randomized polynomial time, polynomial approxiation with constant probability, for any $p$, then P=NP or RP=NP respectivly.
\end{corollary}

\begin{proof}
In both cases we can set $\mathbf{X}_{\alpha} = \langle1, 0, ..., 0 \rangle$ and our objective is exactly the same as it was in the untargeted case, with the only difference being that we need not consider the case when $b_i = 0$ for all $i\le n$ yields $\Phi = 1$, since $\eta = \emptyset$ is an optimal mask. Once we have this setting for $\mathbf{X}_{\alpha}$, the proof follows identically to the proofs of \ref{P=NP} and \ref{RP=NP}.
\end{proof}
\section{Approximation Algorithm for the Additive Case}\label{sec:add}
Our result above shows that polynomial approximations of the optimal solution are intractable in the general case, when the adversary must be able to compute the optimal mask for any prior and any realization of the variables in layer 0. Therefore, we now turn our focus to cases where the DBN exhibits special structure on the transition probabilities.
We start with DBNs with \emph{additive} transition structure, which we define next.
\begin{definition}\label{def:add}
We say a transition probability for $X_i$ is additive if
\[
\mathbb{P}(X_i = 1 | \Pa(X_i)) = \mathbb{P}(X_i = 1| Z_i)
\]
 where  $Z_i = \sum_{X_j^0\in\Pa(X_i^1)}X_j^0$
\end{definition}
We term the problem of finding an optimal adversarial mask when all transitions are additive \emph{ADBNM}, for \emph{Additive DBNM} in the untargeted case, and \emph{TADBNM} refers to the corresponding targeted problem.



\subsection{Inapproximability in the Additive Case}

First, we show that even this case is inapproximable, but now in the sense that no PTAS exists for this problem.
\begin{theorem}\label{prop:add_hard}
No PTAS exists for either ADBNM (untargeted) or TADBNM (targeted), when $p = 1$, unless P=NP, (even for monotone transition functions, when nodes have at most 2 parents).
\end{theorem}
\begin{proof}
To show that no PTAS exists for either problem, we will reduce from Dense $k$-Subgraph (DKSG).
An instance of DKSG is defined by a budget $k$ and a graph $G = (V, E)$. The objective is to find a vertex set $S\subset V$ such that $|\{(u, v)~\in~E~: u, v\in~S\}|$ is maximized while $|S| \leq k$.

To reduce an instance of DKSG to an instance of ADBNM perform the following actions. First, let $\bold{X}^0 = \{X_v^0 : v \in V\}$ and let $\bold{X}^1 = \{ X_{(u, v)}^1: (u, v) \in E\}$. For each $X_v^0\in \bold{X}^0$, let $\mathbb{P}(X_v^0 = 0) = \epsilon$ for arbitrarily small $\epsilon$. It is easy to check that for $\epsilon = \frac{1}{2^{2n}}$, similar reasoning to our previous hardness result holds. Lastly, set $\mathbb{P}(X_{(u, v)}^1| Z_{(u, v)}) = 1$ if $z_{(u, v)} = 2$ and $\mathbb{P}(X_{(u, v)}^1| Z_{(u, v)}) = 0$ otherwise. Suppose that $\bold{x}^0 = <0, 0, ..., 0>$. For TADBNM we need one extra condition that $\mathbf{X}_{\alpha} = \langle1, 1, .., 1 \rangle$.
Now, let $\eta \subset \bold{X}^0$ be any mask. Then, for each pair $X_v^0, X_u^0 \in \eta$, we have 
\[
\mathbb{E}\big[|X_{(u, v)}^1 - X_{\eta, (u, v)}^1|~\big| \bold{x}^0 \big]  = (1-\epsilon)^2
\]
Therefore, for a given $\eta$, the attacker's total utility is
\[
\sum_{X_u^0, X_v^0 \in \bold{X}^1: u\neq v} (1-\epsilon)^2 = \beta(1-\epsilon)^2
\]
where $\beta$ is the number of unique pairs contained in $\eta$.
Hence, the maximum utility an attacker can obtain is $\beta^*(1-\epsilon)^2$ where $\beta^*$ is the maximum number of distinct pairs $X_v^1, X_u^1$ that can be contained in any $\eta$ of size at most $k$. Since each such pair represents an edge in $E$ and $\eta$ represents a collection of vertices of $V$, the maximum dense $k$-subgraph has size $\beta^*$ and is given by the vertices in $\eta$. That is, if a given mask $\eta$ has utility $\beta(1-\epsilon)^2$, then the vertices in $\eta$ correspond to a subgraph of cardinality $\beta$. Similarly, if $S\subset V$ describes a subgraph of size $\beta$, then by mapping the vertices in $S$ to a mask $\eta$, the attacker can achieve utility $\beta(1-\epsilon)^2$.

Since the objectives of the two problems share arbitrary similarity, if a PTAS where to exists for ADBNM, then that same PTAS also exists for DKSG. However, unless P=NP no such algorithm exists for DKSG. Thus, no PTAS exists for ADBNM, unless P=NP.
\end{proof}

\begin{theorem}
For $p \in \mathbb{N}_{\geq 2} \cup \{\infty\}$ ADBNM (untargeted) or TADBNM (targeted), when $p = 1$, unless P=NP, (even for monotone transition functions, when nodes have at most 2 parents).
\end{theorem}

\begin{proof}
    We will use the same reduction from DKSG used in the proof of Theorem \ref{prop:add_hard}. Under construction, and for a general $p$, the attacker's utility for any $\eta$ is 
    \begin{align*}
        \sum_{i = 1}^n \mathbb{P}\big( \sum_{X_j^0 \in \mathbf{X}^0} X_j^0 = i \big)~i^{\frac{1}{p}}
    \end{align*}
    with the understanding that $i^{\frac{1}{\infty}} = 1$.
    Note that this objective function is monotone with respect to the number of unique pairs $X_{u}^0, X_{v}^0 \in \eta$ that correspond to edges $(u, v) \in E$. Further, since each node in $\mathbf{X}^0$ is identical each such pair contributes the same increase to the objective function. Therefore, the objective function increases with respect to the number of unique pairs corresponding to edges in the original graph, independent of which pair is added. Therefore the objective function of the attacker is maximized by finding the largest set of unique pairs $X^0_{u}, X^0_{v}$ which correspond to edges in the graph, this is the exact objective of the original DKSG problem, meaning that a valid solution to one problem is exactly a valid solution to the other and both ADBNM and TADBNM are NP hard for $p > 1$.
\end{proof}

\subsection{Approximation Algorithm}

While even the ADBNM special case is inapproximable in a sense, we now present our first positive result, which is an $n$-approximation (recall that the best known approximation of DKSG is $\Theta(n^{1/4})$, and we showed that our problem is no easier in the reduction above).


First, we impose an additional restriction on the problem: we assume that all transition functions have the propriety that $\mathbb{P}(X_i^1 = 1| Z_i)$ is monotone with respect to $Z_i$. 
We propose Algorithm~\ref{alg:appro2} for this problem.
Next, we show that this algorithm yields a provable approximation guarantee.


\begin{algorithm}
\caption{Approximation algorithm}
\begin{algorithmic}[1]\label{alg:appro2}
\State bestMask := $\emptyset$
\For{\textbf{each} $X_i^1 \in \bold{X}^1$}
    \State $\eta := \emptyset$
    \If{$\mathbb{P}(X_i^1|z_i)$ increasing $\And \mathbb{P}(X_i^1|z^*_i) < \frac{1}{2}$}
        \State $S = \{ X_j^0\in \Pa(X_i^1): x_j^0 = 0\}$
    \ElsIf{$\mathbb{P}(X_i^1|z_i)$ increasing $\And \mathbb{P}(X_i^1|z^*_i) \geq \frac{1}{2}$}
        \State $S = \{ X_j^0\in \Pa(X_i^1): x_j^0 = 1\}$
    \ElsIf{$\mathbb{P}(X_i^1|z_i)$ decreasing $\And \mathbb{P}(X_i^1|z^*_i) < \frac{1}{2}$}
        \State $S = \{ X_j^0\in \Pa(X_i^1): x_j^) = 1\}$
    \ElsIf{$\mathbb{P}(X_i^1|z_i)$ decreasing $\And \mathbb{P}(X_i^1|z^*_i) \geq \frac{1}{2}$}
        \State $S = \{ X_j^0\in \Pa(X_i^1): x_j^0 = 0\}$
    \EndIf
    \While{$|\eta| < k$ and $S\setminus\eta\neq\emptyset$}
        \If{$S$ has outcomes of 1}
            \State $x := \text{argmin}_{s\in S}\mathbb{P}(s = 1)$
        \ElsIf{$S$ has outcomes of 0}
            \State $x := \text{argmax}_{s\in S}\mathbb{P}(s = 1)$
        \EndIf
        
        \State add $x$ to $\eta$
    \EndWhile
    \If{$V(\eta) > V(\text{bestMask})$}
       \State bestMask $:= \eta$
    \EndIf 
\EndFor
\Return bestMask
\end{algorithmic}
\end{algorithm}

\begin{proposition}
For any $p \in \mathbb{N} \cup \{\infty\}$ Algorithm 1 achieves a $n-$approximation on both targeted and untargted attacks. 
\end{proposition}
\begin{proof}
The algorithm generates one mask for each node $X_i^1 \in \mathbf{X}^1$. The associated mask, $\eta_i$, is meant to push the observer's perception of $\mathbb{P}(X_i^1|z_i)$ as close to some extreme (0 or 1) as possible.
We will examine the contribution that the $X_i^1$, most pushed to the desired extreme, makes to the attacker's total utility. Suppose $\mathbb{P}(X_i^1 = 1|z_i^{\eta_i})$ is being pushed to 1. A symmetric argument will hold in the case of 0.
Let 
$    X_a^1 = \arg\max_{X_i}\bigg(\max_{\eta_i} \mathbb{P}\big(X_i = 1|z^{\eta_i}_i \big) \bigg)$   
and let $Q_a = \mathbb{P}(X_i^1 = 1|z_i^{\eta_i})$. 
Next we will show that $Q_a$ is at least $\frac{1}{n}$ of the optimal solution no matter what $L_p$ norm is used. The attacker's utility is given by $\mathbb{E}[||\mathbf{X}_{\eta_i}^1 - \mathbf{X}^1||_p]$, where $\mathbf{X}_{\eta_i} - \mathbf{X}^1$ is a binary vector. For finite $p$ we have, 
\begin{equation*}
    ||\mathbf{X}_{\eta_i} - \mathbf{X}^1||_p = \bigg(\sum_{i = 1}^n |x_{\eta i} - x_i|\bigg)^{\frac{1}{p}}\leq n^{\frac{1}{p}}
\end{equation*}
and in the case when $p = \infty$ we have
\begin{equation*}
    ||\mathbf{X}_{\eta_i} - \mathbf{X}^1||_p = \max_i |x_{\eta i} - x_i|\le 1
\end{equation*}
Under any $p$ the attackers utility on $\eta_a$ is at least $Q_a||1||_p = Q_a$. To get the actual bound on approximation we will split on 3 cases. The first being when $p = 1$, the second being when $2 < p < \infty$ and the third being when $p = \infty$. In each case, each node has probability at most $Q_a$ to attian the desired outcome (0 or 1).
In the first case, when $p = 1$, the attacker's optimal utility is upper-bounded by 
\begin{align*}
    \sum_{i = 1}^n i \binom{n}{i} Q_a^i (1 - Q_a)^{n - i} = n Q_a 
\end{align*}
Hence the ratio to the optimal solution given by $\eta_a$ is 
$    \frac{Q_a}{Q_a n} = \frac{1}{n}.$
In the second case, when $2 < p < \infty$, we have that the attackers optimal utility is upper-bounded by 
\begin{align*}
    &\sum_{i = 1}^n i^{\frac{1}{p}} \binom{n}{i} Q_a^i (1 - Q_a)^{n - i} \\
    \le &\sum_{i = 1}^n i \binom{n}{i} Q_a^i (1 - Q_a)^{n - i} = n Q_a 
\end{align*}
and again we get that the ratio to the optimal solution is $\frac{1}{n}$.

Lastly, when $p = \infty$ the attackers utility is exactly the probability that there exists at least one node with the desired outcome. Since each node has at most probability $Q_a$ to yield the desired outcome, the attacker's optimal utility is at most $ 1 - (1 - Q_a)^n$ and the attacker's utility on $\eta_a$ is at least $Q_a$. Thus the ratio to the optimal solution is at least
$    \frac{Q_a}{1 - (1-Q_a)^n}.$
By montonicity and evaluation of the limit as $Q_a \rightarrow 0$ we see that 
$    \frac{1}{n} \le \frac{Q_a}{1 - (1 - Q_a)^n}.$
Therefore, for any $p\in \mathbb{N} \cup \{\infty\}$ we get an approximation ratio of at least $\frac{1}{n}$

\end{proof}
\subsection{Heuristic}
In addition to our approximation algorithm above, we propose a simple heuristic approach for approximating the optimal mask.
The heuristic is a hill-climbing strategy in which, at each iteration, we add the node to $\eta$ that results in the maximum increase of the value of $\eta$; see Algorithm~\ref{alg:heur}.
As we demonstrate in the experiments below, the combination of the algorithm and the heuristic performs much better than either in isolation (and, of course, jointly achieves the $n$-approximation above).


\begin{algorithm} 
\caption{Heuristic algorithm}
\begin{algorithmic}[1]\label{alg:heur}
\State bestMask := $\emptyset$
\State $\eta$ := $\emptyset$
\While{$|\eta| < k$}
    \State $x$ := node with largest increase to $V(\eta)$
    \State $\eta$ = $\eta \cup \{x\}$
    \If{$V(\eta) > V(\text{bestMask})$}
        \State bestMask = $\eta$
    \EndIf
\EndWhile
\Return bestMask
\end{algorithmic}
\end{algorithm}


We now show that by itself, heuristic can be arbitrarily bad.
Fix $n > 3$ such that $2|n$, let $k = \frac{n}{2}$, and let $p_i = 1 - \epsilon$ for a sufficiently small $\epsilon$. Suppose $\bold{x}^0 = <0, 0, ..., 0>$. Let $\Pa(X_1^1)$ = $\{X_1^0 , X_1^0, ..., X_{n/2}^0\}$, and for each $X_i^1$ with $i > 1$, let 
\\ $\Pa(X_i^1) = \{X_{n/2 + 1}^0, ..., X_n^0 \}$. 
Define $\mathbb{P}(X_1^1 = 1|z_1) = \epsilon z_1$ and for all $i>1$ $\mathbb{P}(X_i^1 = 1|z_i) = 0 $ if $z_i < \frac{n}{2}$, and $\mathbb{P}(X_i^1 = 1|z_i) = 1$ if $z_i = \frac{n}{2}$.
Then we can see that the optimal mask, in both the hiding and flipping case is to hide all nodes $X_{n/2 + 1}^1, ..., X_n^1$. 
Which, results in a value of at least
$\frac{n}{2}(1-\epsilon)^{n/2}$ in the hiding case, and $\frac{n}{2}$ in the flipping case. 
However, since the only way to greedily increase the value of $\eta$ is to keep hiding nodes from $\{X_1^0, ..., X_{n/2}^0\}$, the mask produced by the heuristic will have value $(1-\epsilon)^{n/2}\epsilon\frac{n}{2}$.
Thus,
we get a ratio of 
\[\frac{(1-\epsilon)^{n/2}\epsilon\frac{n}{2}}{\frac{n}{2}(1-\epsilon)^{n/2}}
= \epsilon\]
Note that $\epsilon$ is independent of $n$.
Thus, as $\epsilon \rightarrow 0$ the value of the heuristic solution also converges to $0$ $\forall \ n>3$.

Next we will define and discuss linear Bayesian networks, on such networks this proposed heuristic is guaranteed to find the optimal solution, although doing so can be achieved by a much simpler algorithm which we will also discuss.  

\section{Polynomial-time Algorithm for Linear Bayesian Networks}\label{lin}

Our final contribution is a further restriction on the DBN that yields a polynomial-time algorithm for computing an optimal mask for the adversary.
Specifically, we consider networks in which each transition function is of the form 
$$\mathbb{P}\big(X_i^1 = 1| \Pa(X_i^1) \big) = \sum_{X_j^0\in \Pa(X_i^1)} a_{ij}X_j^0.$$
We call these \emph{linear Bayesian networks}.
\begin{theorem}
In linear Bayesian networks 
the optimal solution to DBNM and TDBNM can be computed in polynomial time for the $l_1$-norm.
\end{theorem}
\begin{proof}
Consider the untargeted case first. Let $\bold{x}^0$ be the outcome given by nature. Let $\bold{y}^0$ be any outcome of $\bold{X}^0$ which agrees with $\bold{x}^0$ on all elements except those in $\eta$. More specifically, if $X_j^0\notin\eta$ then $x_j^0 = y_j^0$ and if $X_j^0 \in \eta$ then $y_j^0$ is free to be either 0 or 1.

For notational convenience we define the following variables for any mask $\eta$, and for any $X_i^1\in \bold{X}^1$ let 
\begin{align*}
P_{i,r} = \Pa(X_i^1)\cap\eta_r&~\text{ and }~P_i = \Pa(X_i^1)\cap\eta \\
Q_i  = \sum_{X_j^0\in \Pa(X_i^1)}a_{ij}x_j^0&~\text{ and }~R_i  = \sum_{X_j^0 \in \Pa(X_i^1)\setminus\eta}a_{ij}x_j^0
\end{align*}
then attacker's utility on $X_1^i$ can be given as
\[Q_i  + R_i+\sum_{X_j^0\in P_i} a_{ij}p_j - 2Q_i\big(R_i + \sum_{X_j^0\in P_i}a_{ij}p_j\big)\]
Consider the change in value of $\eta$ when adding some $X_r^0\in \Pa(X_i^1)\setminus\eta$ denote this new mask as $\eta_r = \eta\cup\{X_r^0\}$. Assume that $x_r^0 = 1$, a symmetric argument will yield a similar result when $x_r^0 = 0$. For notational convenience, let $R'_i = R_i - 1$. Then, the difference in value of $\eta$ and $\eta_r$ is
\begin{align*}
& Q_i  + R'_i + \sum_{X_j^0\in P_{i,r}} a_{ij}p_j - 2Q_i (R'_i + \sum_{X_j^0\in P_{i,r}}a_{ij}p_j) \\
& \quad - Q_i - R_i - \sum_{X_j^0\in P_i} a_{ij}p_j + 2Q_i \big(R_i - \sum_{X_j^0\in P_i}a_{ij}p_j\big) \\
 &=  - a_{ir}p_r(1 - 2Q_i) \\
\end{align*}
Thus for any $X_i^1 \in \bold{X}^1$ if we hide $X_r^0$ when $x_r^0 = 1$, then the change in utility to $X_i^1$'s contribution to the total utility is $-p_ra_{ir}(1-2Q_i)$, and similarly when $x_r^0 = 0$, the change is $p_ra_{ir}(1-2Q_i)$. Thus in both cases we get that hiding $X_r^0$ causes the attacker's utility to increase by $(-1)^{\beta_r}p_ra_{ir}(1-2Q_i)$ where $\beta_r = x_r^0$. In the targeted case the only way in which our analysis changes is in the value of $\beta_r$. Since we now have a desired target for each $X_i^1$, if that desired target is 0 then $\beta_r$ is also 0 and similarly when the target is 1, so is $beta_r$. Thus in both the targeted and untargeted case the change in utility is independent of the current mask $\eta$ and that the total utility is simply the sum of the utility on each $X_i^1$. Thus, when hiding any $X_r^0$ the change in the attacker's total utility increases linearly by a value that depends only on $x_r^0$ and not on the current mask $\eta$. Therefore the attackers utility can be written as 
\[
 \sum_{i = 1}^n Q_i + \sum_{r \in \mathbb{I}(\Pa(X_i^1)}y_r(-1)^{x_r^0}p_r a_{ir}(1-2Q_i)
 \]
where $\mathbb{I}(\Pa(X_i^1)$ is the index set of the parents of $X_i^1$, and if $X_r^0 \in \eta$ then $y_r = 1$ and if $X^0_r \notin \eta$ then $y_r = 0$. Assigning values to each $y_r$ such that $\sum_{r = 1}^n y_r \leq k$ can be done in polynomial time by simply selecting the $y_r$'s with the highest associated coefficients.
\end{proof}

\section{Experiments}\label{sec:exp}


As discussed in Section \ref{sec:add}, our approximation scheme is to compute both the $n$-approximation mask and the heuristic mask, then take the one yielding the higher utility. 
Note that this combination clearly yields an $n$-approximation.
As we now demonstrate, it is also significantly better in combination than either of the approaches by itself.



%

\begin{figure}[h]
    \centering
    \includegraphics[scale=0.2]{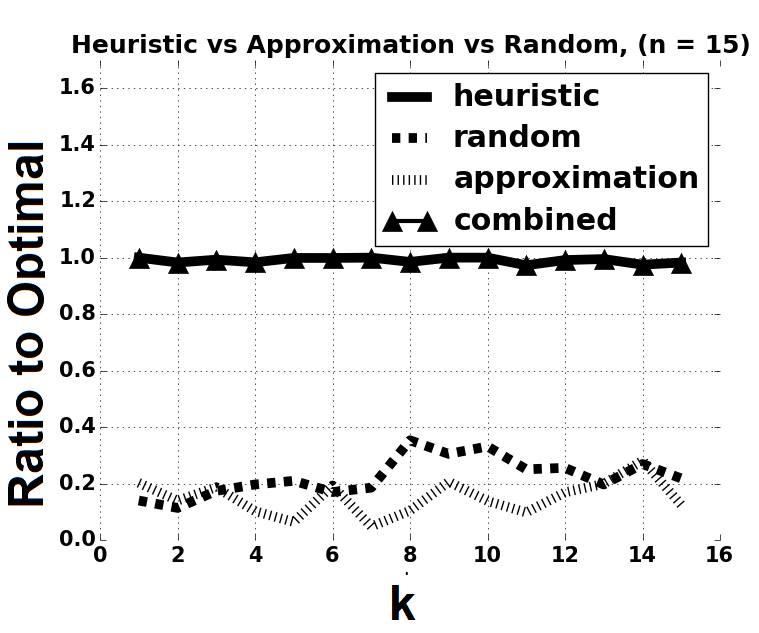}
    \includegraphics[scale=0.2]{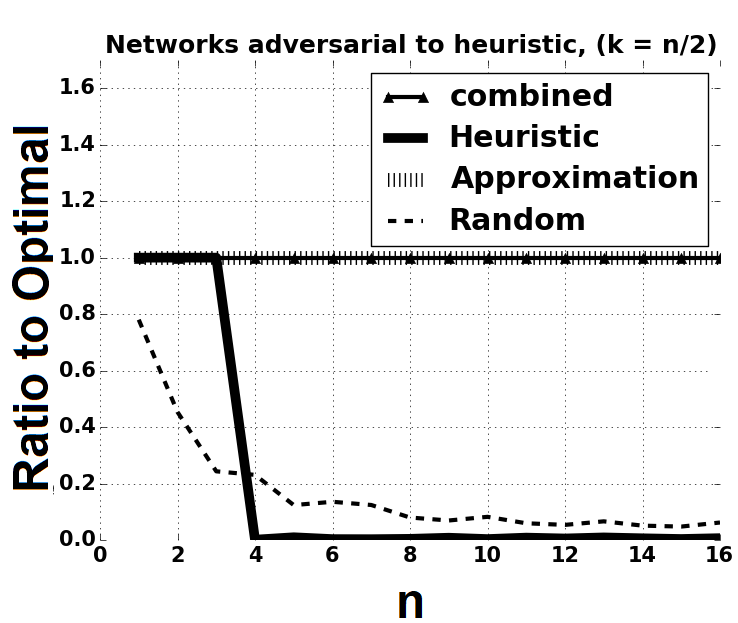}
    \caption{Comparison between our combined algorithm, heuristic and approximation algorithms in isolation, and random masking on randomly generated networks (left) and networks generated adversarially (right).}
    \label{fig:net_compt}
\end{figure}

Figure \ref{fig:net_compt} (left) shows the results on random general and additive networks, and demonstrates that our combined algorithm significantly outperforms the approximation algorithm, largely on the strength of the heuristic, which is highly effective in these settings.
Figure \ref{fig:net_compt} (right) studies settings constructed to be adversarial to the heuristic.
As we can see, here the combined algorithm performs similarly to the approximation algorithm, while the heuristic in isolation ultimately performs 
poorly.
Thus, the combination of the two is far stronger than each component in isolation.

\section{Conclusion}
We introduce a model of deception in which a principal needs to make a decision based on the state of the world, and an adversary can mask information about the state. We study this in a model where the principal is oblivious to the presence of the adversary and reasons about state change using a dynamic Bayes network. Even in a simple two time period model, we showt the existence of cases where an adversary with the ability to mask information about the state at time 0 can cause the oblivious principal to  have an arbitrarily incorrect posterior.
However, computing, or even approximating these masks to within a polynomial factor, is NP-hard in the general case.
We also consider this problem with special structure on the transition probabilities, showing that when transitions only depend on the sum of parent values, the problem remains inapproximable, although we now exhibit an $n$-approximation.
On the other hand, when transitions are linear, we show that it can be solved in polynomial time.

\small
\bibliographystyle{aaai}
\bibliography{ht}



\end{document}